\documentclass[a4paper]{article}
\usepackage[utf8]{inputenc}

\usepackage[a4paper,top=3cm,bottom=2cm,left=3cm,right=3cm,marginparwidth=1.75cm]{geometry}

\usepackage{amsmath,amsfonts,dsfont,amsthm}
\usepackage{graphicx}
\usepackage[colorinlistoftodos]{todonotes}
\usepackage[colorlinks=true, allcolors=blue]{hyperref}
\usepackage{float}
\usepackage{multirow}

\usepackage{microtype}
\usepackage{subfigure}
\usepackage{booktabs}
\usepackage{algorithm}
\usepackage{algorithmic}

\usepackage{hyperref}

\newcommand\numberthis{\addtocounter{equation}{1}\tag{\theequation}}

\newtheorem{thm}{Theorem}

\newtheorem{rem}{Remark}

\allowdisplaybreaks

\title{AutoShuffleNet: Learning Permutation Matrices \\via an Exact Lipschitz Continuous Penalty \\in Deep Convolutional Neural Networks}
\author{Jiancheng Lyu $^{\dagger}$, \and Shuai Zhang $^*$,\and Yingyong Qi \thanks{Qualcomm AI Research, San Diego, CA 92121. Email: (shuazhan,yingyong)@qti.qualcomm.com.} , \and Jack Xin
\thanks{Department of Mathematics, UC Irvine, Irvine, CA 92697.
Email: (jianchel, jack.xin)@uci.edu.}
}
\date{}

\begin{document}

\maketitle

\vskip 0.3in

\begin{abstract}
ShuffleNet is a state-of-the-art light weight convolutional neural network architecture. Its basic operations include group, channel-wise convolution and channel shuffling. However, channel shuffling is manually designed empirically. Mathematically, shuffling is a multiplication by a permutation matrix. In this paper, we propose to automate channel shuffling by learning permutation matrices in network training. We introduce an exact Lipschitz continuous non-convex penalty so that it can be incorporated in the stochastic gradient descent to approximate permutation at high precision. Exact permutations are obtained by simple rounding at the end of training and are used in inference. The resulting network, referred to as AutoShuffleNet, achieved improved classification accuracies on CIFAR-10 and ImageNet data sets. In addition, we found experimentally that the standard convex relaxation of permutation matrices into stochastic matrices leads to poor performance. We prove theoretically the exactness (error bounds) in recovering permutation matrices when our penalty function is zero (very small). We present examples of permutation optimization through graph matching and two-layer neural network models where the loss functions are calculated in closed  analytical form. In the examples, convex relaxation failed to capture permutations whereas our penalty succeeded.  

\end{abstract}

\section{Introduction}
Light convolutional deep neural networks (LCNN) are attractive in resource limited conditions for delivering high performance at low costs. Some of the state-of-the-art LCNNs in computer vision are ShuffleNet (\cite{shufflev1}, \cite{shufflev2}), IGC (Interleaved Group Convolutions, \cite{IGC}) and IGCV3 (Interleaved Low-Rank Group Convolutions,\cite{IGCV}). A noticeable feature in the design is the presence of permutations for channel shuffling in between separable convolutions. The permutations are hand-crafted by designers outside of network training however. A natural question is whether the permutations can be learned like network weights so that they are optimized based on training data. An immediate difficulty is that unlike weights, permutations are highly discrete and incompatible with the stochastic gradient descent (SGD) methodology that is continuous in nature. To overcome this challenge, we introduce an exact Lipschitz continuous non-convex penalty so that it can be incorporated in SGD to approximate permutation at high precision. Consequently, exact permutations are obtained by simple rounding at the end of network training with negligible drop of classification accuracy. To be specific, we shall work with ShuffleNet architecture (\cite{shufflev1}, \cite{shufflev2}). Our approach extends readily to other LCNNs.

{\bf Related Work.} Permutation optimization is a long standing problem arising in operations research, graph matching among other applications \cite{KB_57,Burk_13}. Well-known examples are linear and quadratic assignment problems \cite{vog}. Graph matching is a special case of quadratic assignment problem, and can be formulated over $N\times N$ permutation matrices $\mathcal{P}^N$ as: $\min_{\pi \in \mathcal{P}^N} \| A - \pi B \pi^T \|^{2}_{F}$, where $A$ and $B$ are the adjacency matrices of graphs with $N$ vertices, and $\|\cdot\|_F$ is the Frobenius norm. A popular way to handle $\mathcal{P}^N$ is to relax it to the Birkhoff polytope $\mathcal{D}^N$, the convex hull of $\mathcal{P}^N$, leading to a convex relaxation. The problem may still be non-convex due to the objective function. The explicit realization of $\mathcal{D}^N$ is the set of doubly stochastic matrices $\mathcal{D}^N =\{M\in \mathbf{R}^{N\times N}: M\mathbf{1}=\mathbf{1}, M^T \mathbf{1} = \mathbf{1}, M \geq 0 \}$, where $\mathbf{1}=(1,1,,\cdots,1)^T \in \mathbf{R}^{N}$. An approximate yet simpler way to treat $\mathcal{D}^N$ is through the classical first order matrix scaling algorithm e.g. the Sinkhorn method \cite{Sinkhorn}. Though in principle such algorithm converges, the cost can be quite high when   iterating many times, which causes a bottleneck effect \cite{wright_16}. A non-convex and more compact relaxation of $\mathcal{P}^N$ is by a sorting network \cite{wright_16} which maps the box $[0,1]^N$ into a manifold that sits inside $\mathcal{D}^N$ and contains $\mathcal{P}^N$. Yet another method is path following algorithm \cite{path_09} which seeks solutions under concave relaxations of $\mathcal{P}^N$ by solving a linear interpolation of convex-concave problems (starting from the convex relaxation). None of the existing relaxations are exact. 

{\bf Contribution.} Our non-convex relaxation is a combination of matrix $\ell_{1-2}$ penalty function and $\mathcal{D}^N$. The $\ell_{1-2}$ (the difference of $\ell_1$ and $\ell_2$ norms) has been proposed and found effective in selecting sparse vectors under nearly degenerate linear constraints \cite{ELX_13,DL12}. The matrix version is simply a sum of $\ell_{1-2}$ over all row and column vectors. We prove that {\it the penalty is zero when applied to a matrix in $\mathcal{D}^N$ if and only if the matrix is a permutation matrix.} Thanks to the $\mathcal{D}^N$ constraint, the penalty function is Lipschitz continuous (almost everywhere differentiable). This allows the penalty to be integrated directly into SGD for learning permutation in LCNNs. As shown in our experiments on CIFAR-10 and Imagenet data sets, the closeness to $\mathcal{P}^N$ turns out to be remarkably small at the end of network training so that a simple rounding has negligible effect on the validation accuracy. We also found that convex relaxation by $\mathcal{D}^N$ fails to capture good permutations for LCNNs. To our best knowledge, this is {\it the first time permutations have been successfully learned on deep CNNs to improve hand-crafted permutations}.

{\bf Outline.} In section 2, we introduce our exact permutation penalty, and prove the closeness to permutation matrices when the penalty values are small as observed in the experiments. We also present the training algorithm combining thresholding and matrix scaling to approximate projection onto $\mathcal{P}^N$ for SGD. In section 3, we analyze three permutation optimization problems to show the necessity of our penalty. In a 2-layer neural network regression model with short cut (identify map), convex relaxation does not give the optimal permutation even with additional rounding while our penalty can. In section 4, we show experimental results on consistent improvement of auto-shuffle over hand-crafted shuffle on both CIFAR-10 and Imagenet data sets. Conclusion is in section 5.

\section{Permutation, Matrix \texorpdfstring{$\ell_{1-2}$}{Lg} Penalty and Exact Relaxation}

The channel shuffle operation in ShuffleNet \cite{shufflev1,shufflev2} can be represented as multiplying the feature map in the channel dimension by a permutation matrix $M$. The permutation matrix $M$ is a square binary matrix with exactly one entry of one in each row and each column and zeros elsewhere. In the ShuffleNet architecture \cite{shufflev1,shufflev2}, $M$ is preset by the designers and will be called ``manual''. In this work, we propose to learn an automated permutation matrix $M$ through network training, hence removing the human factor in its selection towards a more optimized shuffle. Since permutation is discrete in nature and too costly to enumerate, we propose to approach it by adding a matrix generalization of the $\ell_{1-2}$ penalty \cite{ELX_13,DL12} to the network loss function in the stochastic gradient descent (SGD) based training.

Specifically for $M = \left(m_{ij}\right) \in \mathbf{R}^{N\times N}$, the proposed continuous matrix penalty function is
\begin{align}\label{DL12}
    P\left(M\right) := \sum_{i=1}^N\left[\sum_{j=1}^N\left|m_{ij}\right|-\left(\sum_{j=1}^Nm_{ij}^2\right)^{1/2}\right]+\sum_{j=1}^N\left[\sum_{i=1}^N\left|m_{ij}\right|-\left(\sum_{i=1}^Nm_{ij}^2\right)^{1/2}\right],
\end{align}
in conjunction with the doubly stochastic constraint: 
\begin{equation}{\label{ct}}
m_{ij}\geq 0,\; \forall (i,j); \; \; \sum_{i=1}^{N}\, m_{ij} = 1,\; \forall \, j; \; \sum_{j=1}^{N} \, m_{ij} = 1, \; \forall \, i.
\end{equation}


\begin{rem}
When the constraints in \eqref{ct} hold, $\sum_{j=1}^N\left|m_{ij}\right|$ and 
$\sum_{i=1}^N\left|m_{ij}\right|$ in $P\left(M\right)$ can be removed. However, in actual computation, the two equality constraints of \eqref{ct} only hold approximately, so the full expression in \eqref{DL12} is necessary.
\end{rem}

\begin{rem}
Thanks to \eqref{ct}, we see that the penalty function $P(M)$ is actually Lipschitz continuous in $M$ as $\sum_{j=1}^N\,  m_{ij}^{2}\not =0$, $\forall i$, and $\sum_{i=1}^N\, m_{ij}^{2} \not =0$, 
$\forall j$. 
\end{rem}

\begin{thm}
A square matrix $M$ is a permutation matrix if and only if $P(M)=0$, and the doubly stochastic constraint \eqref{ct} holds. 
\end{thm}

\begin{proof} ($\Rightarrow$) Since a permutation matrix consists of columns (rows) with exactly one entry of 1 and the rest being zeros,  each term inside the outer sum of $P(M)$ equals zero, and clearly (\ref{ct}) holds. 
 
($\Leftarrow$) By the Cauchy-Schwarz inequality, 
\[ \left ( \sum_{j=1}^N\left|\, m_{ij}\, \right|\right ) -\left(\sum_{j=1}^N m_{ij}^2\right)^{1/2} \geq 0, \;\; \forall i,\]
with equality if and only if the row-wise cardinalty is 1:
\begin{equation}\label{card1}
|\, \{j: m_{ij} \not = 0\}\, | = 1, \;\; \forall i.
\end{equation}
This is because the mixed product terms like $2\, |m_{ij}\, m_{ij'}|$ ($j\not = j'$)  in $(\, \sum_{j=1}^N\left|\, m_{ij}\, \right|\, )^2$ must all be zero to match $\sum_{j=1}^N \, m_{ij}^2$. This only happens when equation (\ref{card1}) is true.  Likewise,
\[\sum_{i=1}^N\left|\,m_{ij}\right|-\left(\sum_{i=1}^Nm_{ij}^2\right)^{1/2} \geq 0, \;\; \forall j,\]
with equality if and only if
$ | \, \{i: m_{ij} \not = 0 \}\, | = 1, \;\; \forall j $.
In view of (\ref{ct}), $M$ is a permutation matrix.
\end{proof}

The non-negative constraint in (\ref{ct}) is maintained throughout SGD by thresholding $m_{ij} \rightarrow \max(m_{ij},0)$. The normalization conditions in (\ref{ct}) are implemented sequentially once in SGD iteration. Hence they are not strictly enforced. In theory, if the column normalization (divide each column by its sum) and row normalization (divide each row by its sum) are iterated sufficiently many times, the resulting matrices converge to (\ref{ct}). This is known as the Sinkhorn process or RAS method \cite{Sinkhorn}, which is a first order method to approximately solve the so called matrix scaling problem (MSP).  Simply state, the MSP for a given non-negative real matrix $A\in \mathbf{R}^{N\times N}$ is to scale its rows and columns (i.e. multiply each by a non-negative constant) to realize the prescribed row sums and column sums. The approximate MSP is: given tolerance $\epsilon > 0$, find positive diagonal matrices $X$ and $Y$ such that $ | \, X A Y \mathbf{1} - \mathbf{1}\, | \leq \epsilon, \; |\, \mathbf{1}^{T} X A Y - \mathbf{1}^{T} \, | \leq \epsilon. $
For a historical account of MSP and a summary of various algorithms to date, see \cite{ZW17} and Table 1 therein. The RAS method is an alternate minimization procedure with convergence guarantees. Each iteration of the RAS method costs complexity $O(m)$, $m$ being the number of non-zero entries 
in $A$. If the entries of $A$ are polynomially bounded (which is the case during network training due to the continuous nature of SGD), the RAS method converges in $\tilde{O}(N/\epsilon^2)$ iterations \cite{RAScomplex}, giving total complexity $\tilde{O}(m N/\epsilon^2)$, where tilde hides logarithmic factors in $N$ and $\epsilon^{-1}$. Improvements of complexity bounds via  minimizing the log of capacity and higher order methods can be found in \cite{ZW17}. However for our study, the first order method \cite{Sinkhorn} suffices for two reasons. One is that it is computationally low cost, the other is that the error in the matrix scaling step can be compensated in network weight adjustment during SGD. In fact, we did not find much benefit to iterate RAS method more than once in terms of enhancing validation accuracy. This is quite different from solving MAP as a stand alone task.  


The multiplication by $M$ can be embedded in the network as a $1\times1$ convolution layer with $M$ initialized as absolute value of a random Gaussian matrix. After each weight update, we threshold the weights to $\left[0,\infty\right)$, normalize rows to unit lengths, then repeat on columns. Let $L$ be the network loss function. The training minimizes the objective function:
\begin{equation}\label{obj}
f=f(w,M):= L(w) + \lambda \, \sum_{j=1}^{J} P(M_j),
\end{equation}
where $J$ is the total number of ``channel shuffles'' $M_j$'s abbreviated as $M$ , $w$ is the network weight, $\lambda$ a positive parameter. The training algorithm is summarized in Algorithm \ref{alg}. The $\ell_1$ term in the penalty function $P$ has standard sub-gradient, and the $\ell_2$ term is differentiable away from zero, which is guaranteed in the algorithm \ref{alg} due to continuity of SGD and the normalization in columns and rows. $\lambda$ is chosen to be $10^{-3}$ or $2\times10^{-3}$ so as to balance the contributions of the two terms in \eqref{obj} and drive $\sum_{j=1}^{J} P(M_j)$ close to $0$.


\begin{algorithm}
\caption{AutoShuffle Learning.}\label{alg}
\textbf{Input}: \\ mini-batch loss function $f_t(w,M)$, $t$ being the iteration index;\\
        learning rate $\eta^t$ for $(w,M)$;\\
        penalty parameter $\lambda$ for $P$;\\
        total iteration number $Tn$.\\    
        \textbf{Start}:\\
        $w$: sample from unit Gaussian distribution;\\
        $M$: sample from unit Gaussian distribution then take absolute value.\\
\textbf{WHILE $t<Tn$, DO}:
\begin{algorithmic}
    \STATE (1) Evaluate the mini-batch gradient $(\nabla_w f_t, \nabla_{M} f_t)$ at $(w^t, M^t)$;
    \STATE (2) $w^{t+1} = w^t - \eta^t \, \nabla_w f_t(w^t,M^t)$; \quad $//$ gradient update for weights
    \STATE (3) $M^{t+1} = M^t - \eta^t\,   \nabla_{M} f_t(w^t,M^t)$; 
    \quad $//$ gradient update for $M$
    \STATE (4) $M^{t+1} \leftarrow \max(M^{t+1},0)$; \quad $//$ thresholding to enforce non-negativity constraint
  \STATE (5) normalize each column of $M^{t+1}$ by dividing the sum of entries in the column; 
  \STATE (6) normalize each row of $M^{t+1}$ by dividing the sum of entries in the row. \\
 \textbf{END WHILE}\\
\textbf{Output}: $w^{Tn}$,$M^{Tn}$;
project each matrix $M_{j}^{Tn}$ inside $M^{Tn}$ to the nearest permutation matrix.
\end{algorithmic}
\end{algorithm}
\medskip

We shall see that the penalty $P$ indeed gets smaller and smaller during training. Here we show a theoretical bound on the distance to $\mathcal{P}^N$ when $P$ is small and (\ref{ct}) holds approximately. 

\begin{thm}
Let the dimension $N$ of a non-negative square matrix $M$ be fixed. If $P(M) = O(\epsilon)$, $\epsilon \ll 1$,  and the doubly stochastic constraints are satisfied to $O(\epsilon)$, then there exists a permutation matrix $P^*$ such that $\|\, M - P^*\, \| =O(\epsilon)$.
\end{thm}

\begin{proof} It follows from $P(M) = O(\epsilon)$ that 
\[ \left ( \sum_{j=1}^N\left|\, m_{ij}\, \right|\right ) -\left(\sum_{j=1}^N m_{ij}^2\right)^{1/2} =O(\epsilon), \;\; \forall i,\]
implying that:
\begin{equation}\label{approxn1}
|\, m_{ij}\, m_{ij'}\, |  = O(\epsilon), \;\; \forall j\not = j', \;\; \forall i. 
\end{equation}
On the other hand for $\forall i$:
\begin{equation}\label{approxn2}
\sum_{j=1}^{N}\, m_{ij} = 1 + O(\epsilon). 
\end{equation}
Let $j^* ={\rm argmax}_{j} |m_{ij}|$, at any $i$. It follows from (\ref{approxn2}) that 
\[ |\, m_{ij^*}| \geq 1/N + O(\epsilon), \]
and from (\ref{approxn1}) that 
\[ m_{ij'} = O(\epsilon), \;\;\; \forall j'\not = j^*.\]
Hence each row of $M$ is $O(\epsilon)$ close to a unit coordinate vector, with one entry near 1 and the rest near 0. Similarly from $\sum_{i=1}^N\left|\,m_{ij}\right|-\left(\sum_{i=1}^Nm_{ij}^2\right)^{1/2} = O(\epsilon), \;\; \forall j$, and $ \sum_{i=1}^{N}\, m_{ij} = 1 + O(\epsilon)$, we deduce that each column of $M$ is $O(\epsilon)$ close to a unit coordinate vector, with one entry near 1 and the rest near 0. Combining the two pieces of information above, we conclude that $M$ is $O(\epsilon)$ close to a permutation matrix.
\end{proof}

The learned non-negative matrix $M$ will be called a {\it relaxed shuffle} and will be rounded to the nearest permutation matrix to produce a final {\it auto shuffle}. Strictly speaking, this ``rounding'' involves finding the orthogonal projection to the set of permutation matrices, a problem called the linear assignment problem (LAP), see \cite{Kimm_15} and references therein. The LAP can be formulated as a linear program over the doubly stochastic matrices or constraints (\ref{ct}), and is solvable in polynomial time \cite{Kimm_15}. As we shall see later in Table \ref{tab:3}, the relaxed shuffle comes amazingly close to an exact permutation in network learning. Hence, it turns out unnecessary to solve LAP exactly, indeed a simple rounding will do. AutoShuffleNet units adapted from ShuffleNet v1 \cite{shufflev1} and ShuffleNet v2 \cite{shufflev2} are illustrated in Figs. \ref{unit1}-\ref{unit2}.

\begin{figure}[!ht]
\centering
\begin{minipage}[b]{0.47\textwidth}
\begin{center}
\centerline{
    \includegraphics[width=0.445\textwidth]{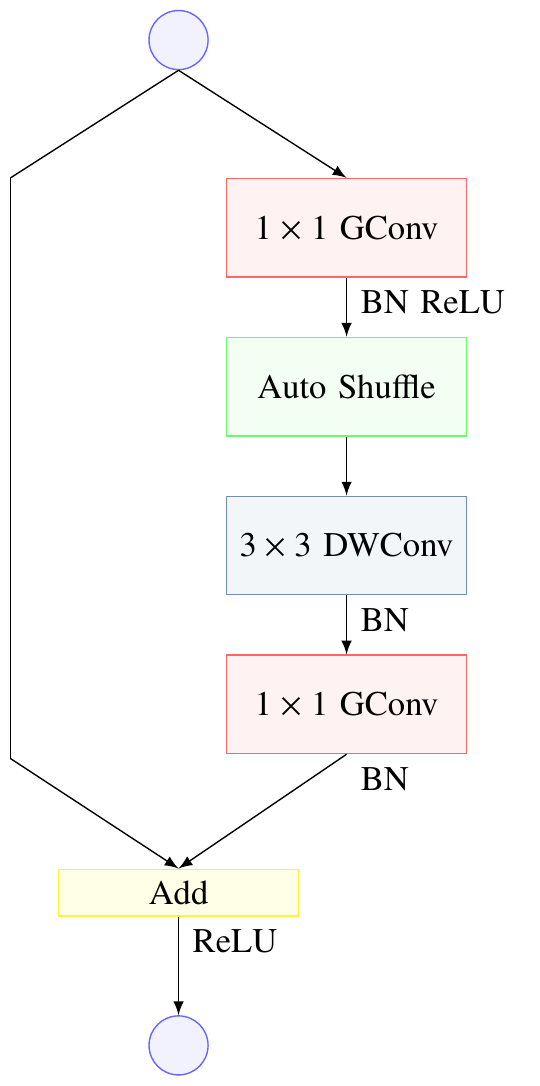}
    \includegraphics[width=0.555\textwidth]{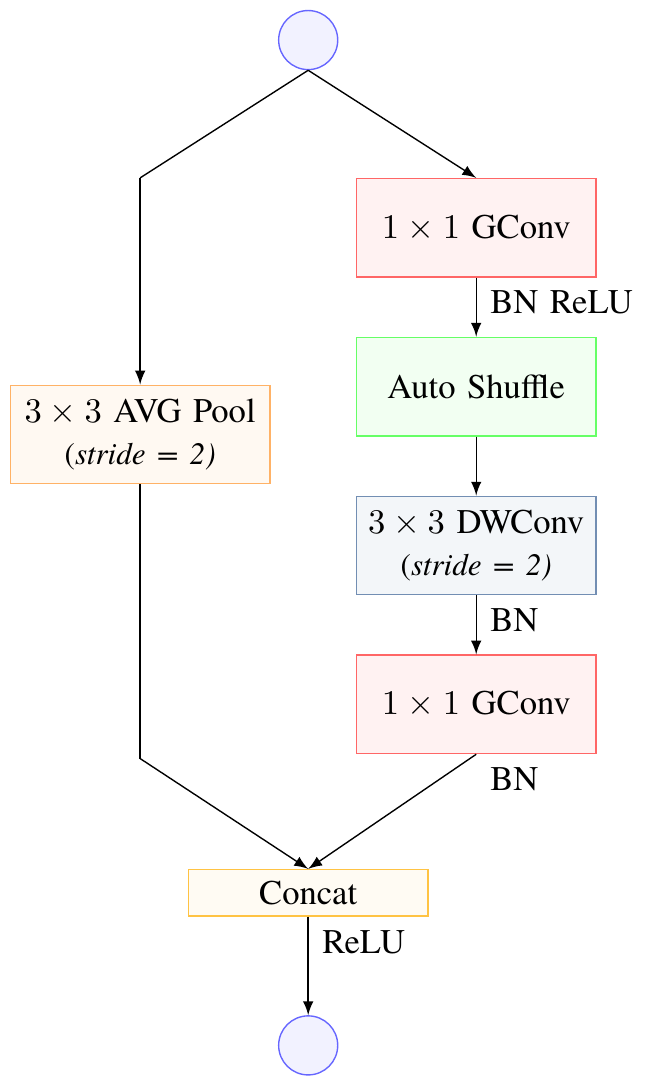}
    }
\caption{AutoShuffleNet units based on ShuffleNet v1.}
\label{unit1}
\end{center}
\end{minipage}
\hfill
\begin{minipage}[b]{0.47\textwidth}
\begin{center}
 \centerline{
    \includegraphics[width=0.45\textwidth]{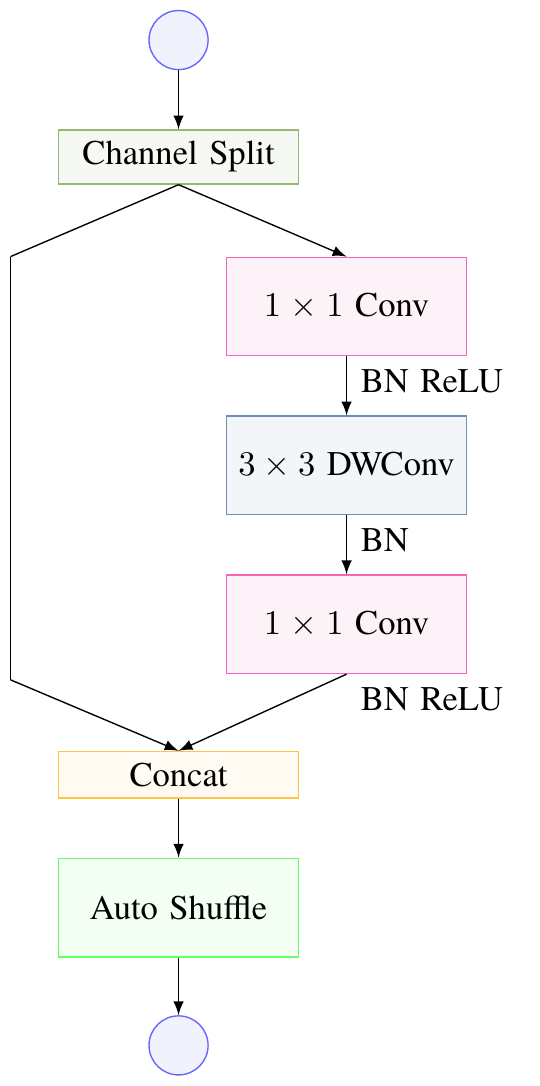}
    \includegraphics[width=0.55\textwidth]{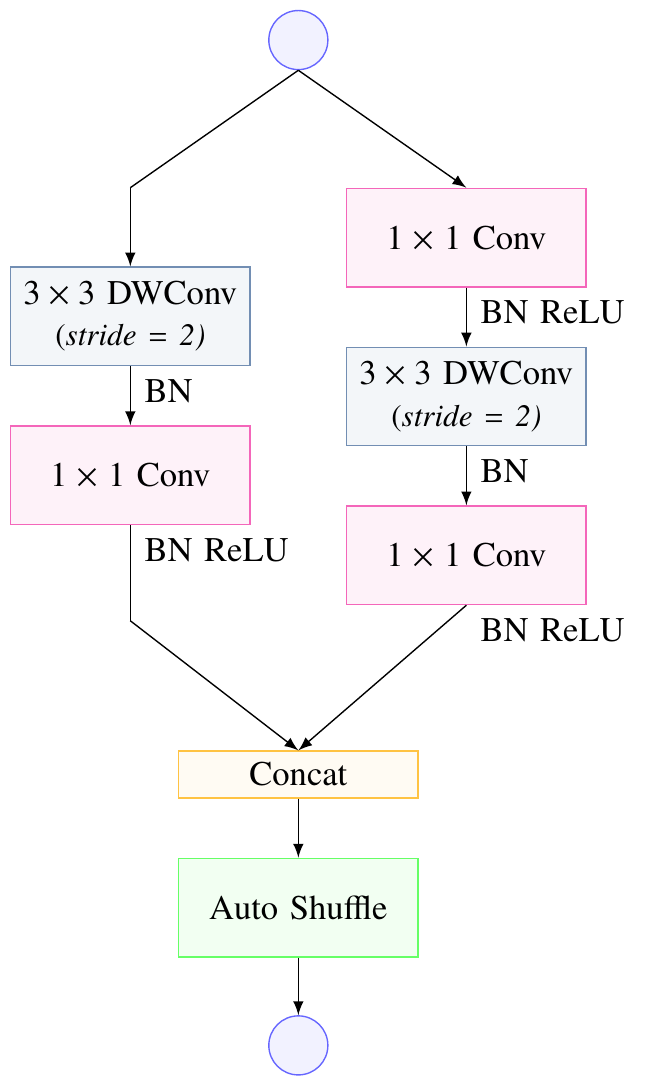}
    }
\caption{AutoShuffleNet units based on ShuffleNet v2.}
\label{unit2}   
\end{center}
\end{minipage}
\end{figure}

\section{Permutation Problems Unsolvable by Convex Relaxation}
The doubly stochastic matrix condition (\ref{ct}) is a popular convex relaxation of permutation. However, it is not powerful enough to enable auto-shuffle learning as we shall see later. In this section, we present examples from permutation optimization to show the limitation of convex relaxation (\ref{ct}), and how our proposed penalty (\ref{DL12}) can strength (\ref{ct}) to retrieve permutation matrices. 

Let us recall the graph matching (GM) problem, see \cite{vog,wright_16,Kimm_15,anticvx,path_09} and references therein. The goal is to align the vertices of two graphs to minimizes the number of edge disagreements. Given a pair of $n$-vertex graphs $G_A$ and $G_B$, with respective adjacency $n\times n$ matrices $A$ and $B$, the GM problem is to find a permutation matrix $Q$ to minimize $\| A Q - Q B \|_{F}^{2}$. Let $\Pi$ be the set of all permutation matrices, solve 
\begin{equation}\label{gm1}
Q^*:= {\rm argmin}_{Q \in \Pi} \, \| A\, Q - Q\, B \|_{F}^{2}.
\end{equation}
By algebraic identity:
\begin{align*}
    \| A\, Q - Q\, B \|_{F}^{2}  =&\  {\rm trace}\{(A\, Q-Q\, B)^T(A\, Q-Q\, B)\}\\
    =&\ {\rm trace}(A^T\, A) + {\rm trace}(B^T\, B) -2 {\rm trace}(A\, Q\, B^T\, Q^T),
\end{align*}
the GM problem (\ref{gm1}) is same as:
\begin{equation}\label{gm2}
Q^*= {\rm argmin}_{Q\in \Pi} {\rm trace}((-A)\, Q\, B^T\, Q^T),
\end{equation}
a quadratic assignment problem (QAP). The general QAP for two real square matrices $A$ and $B$ is \cite{vog,wright_16}:
\begin{equation}\label{qap}
Q^*= {\rm argmin}_{Q\in \Pi}\; {\rm trace}( A \, Q\, B^T\, Q^T).
\end{equation}
The convex relaxed GM is:
\begin{equation}\label{cvxgm}
Q_*:= {\rm argmin}_{Q \in D^N} \, \| A\, Q - Q\, B \|_{F}^{2}.
\end{equation}
As an instance of general QAP, let us consider problem (\ref{gm1}) in case $n=2$ for two real matrices:
\[ A=\left [ \begin{array}{cc} 
    a & b \\
    c & d 
    \end{array} \right ], \;\; 
    B=\left [ \begin{array}{cc} 
    a' & b' \\
    c' & d' 
    \end{array} \right ].
 \]
If $Q \in \mathcal{D}^2$, then:
\[ Q = \left [ \begin{array}{cc}
    q & 1-q \\
    1-q & q
    \end{array} \right ], \;\;q \in [0,1];
\]
and $A Q - Q B$ equals ( $q':=1-q$ ):
\[\left [ \begin{array}{cc} 
    (a-a')\, q + (b-c')\, q' & (b-b')\, q+(a-d')\, q' \\
    (c-c')\, q + (d-a')\, q' & (d-d')\,q +(c-b')\, q'
    \end{array} \right ]. 
\]

{\bf Example 1:} We show that $Q_*\not = Q^*$ by selecting:
\[ A=\left [ \begin{array}{cc} 
     1 & 2 \\
     3 & 1 
     \end{array} \right ], \;\; 
     B=\left [ \begin{array}{cc} 
     0 & 2 \\
     3 & 1 
     \end{array} \right ].
 \]
\[ A Q- Q B=\left [ \begin{array}{cc} 
    2q-1 & 0 \\
    1-q & 1-q 
    \end{array} \right ],
\]
\[ \|AQ-QB\|_{F}^{2} = (2q-1)^2 + 2 (1-q)^2 = 6 q^2 - 8 q + 3,
\]
which is convex on $[0,1]$ and has minimum at $q_* = 2/3$. The convex relaxed matrix solution is:
\[ Q_* = \left [ \begin{array}{cc} 
    2/3 & 1/3 \\
    1/3 & 2/3 
    \end{array} \right ],
\]  
however, the permutation matrix solution $Q^*$ to GM problem (\ref{gm1}) is the $2\times 2$ identity matrix at $q=1$.  

In the spirit of objective function (\ref{obj}), let us minimize:
\[ \|A Q - QB \|_{F}^{2} + \lambda\, P(Q), \]
or equivalently minimize (after skipping some additive constants in $P$): 
\[ F=F(q):= 6\, q^2 - 8 \, q + 2 - 4 \lambda (q^2 + (1-q)^2)^{1/2}.\]
An illustration of $F$ is shown in Fig. \ref{gmpic}. The minimal point moves from the interior of the interval $[0,1]$ when $\lambda=0.25$ (dashed line, top curve) to the end point 1 as $\lambda$ increases to 1 (line-star, middle curve) and remains there as $\lambda$ further increases to 2 (line-circle, bottom curve). So for $\lambda \in [1,2]$, $Q^*$ is recovered with our proposed penalty. \qed
         
\begin{figure}[!ht]
\begin{center}
\centerline{\includegraphics[width=0.65\textwidth]{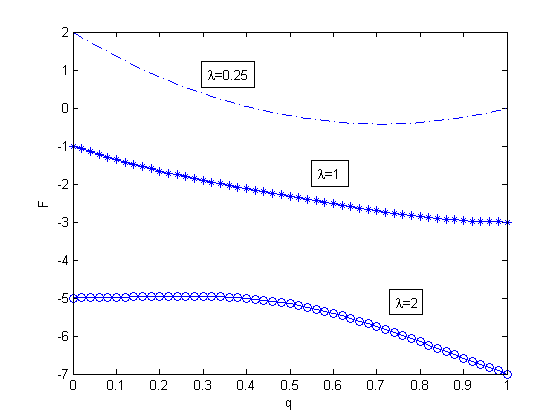}}
\caption{The function $F(q)$ as penalty parameter $\lambda$ varies from 0.25 (interior minimal point, dashed line, top) to 1 (line-star, middle) and 2 (line-circle, bottom). Minimal point occurs at $q=1$ in the latter two curves.}
\label{gmpic}
\end{center}
\end{figure}


{\bf Example 2:} Consider the adjacent matrix $B$ ($A$) of an un-directed graph of 2 nodes and 1 edge (with a loop at node 1). An edge adds 1 and a loop adds 2 to an adjacent matrix.   
\[ A=\left [ \begin{array}{cc} 
    2 & 1 \\
    1 & 0 
    \end{array} \right ], \;\; 
    B=\left [ \begin{array}{cc} 
    0 & 1 \\
    1 & 0 
    \end{array} \right ].
\]
Then:
\[ A Q- Q B=     \left [ \begin{array}{cc} 
    2q & 2(1-q) \\
    0 &  0 
    \end{array} \right ],
\]
\[ \|A\, Q-Q\, B\|_{F}^{2} = 4[q^2 + (1-q)^2]. \]
So $Q_* = Q(1/2) \not = Q^* = Q(0)=Q(1)$. The $P$ regularized objective function (modulo additive constants) is:
\[ F=4 [q^2 + (1-q)^2] - 4 \lambda (q^2 + (1-q)^2)^{1/2}, \]
with $F(0)=F(1)= 4 - 4 \lambda$. In view of:
\[ F'/4 = (2q-1)[ 2 - \lambda /(q^2 + (1-q)^2)^{1/2}], \]
two possible interior critical points are:
\begin{equation}\label{2crit}
q = 1/2 \;\;{\rm or}\;\;  q^2 + (1-q)^2 = \lambda^2/4.
\end{equation}
Since $\max_{q\in [0,1]}\{ q^2 + (1-q)^2\}= 1$, if $\lambda > 2$, the second equality in (\ref{2crit}) is ruled out. Comparing $F(1/2) = 2 - 4 \lambda 2^{-1/2} = 
2 (1- \sqrt{2} \lambda )$ with $F(0)$, we see that the global minimal point does not occur at $q=1/2$ if
\[ 1- \sqrt{2} \lambda > 2 - 2 \lambda\;\; {\rm or}\;\; 
 \lambda > 1/(2-\sqrt{2}) \approx 1.7071.\]
Hence if $\lambda > 2$, minimizing $F$ recovers $Q^*$.  \qed

In Fig. \ref{gmpic2}, we show that two minimal points of $F$ occur in the interior of $(0,1)$ when $\lambda =1.8, 1.9$, and transition to $q=0,1$, at $\lambda = 2$. When $\lambda^2/4 < \min_{q\in [0,1]} \{q^2 + (1-q)^2\} = 1/2$, or $\lambda < \sqrt{2}$, the second equality in (\ref{2crit}) cannot hold, $F$ becomes convex with a unique minimal point at $q=1/2$. 

\begin{figure}[!ht]
\begin{center}
\centerline{\includegraphics[width=0.65\textwidth]{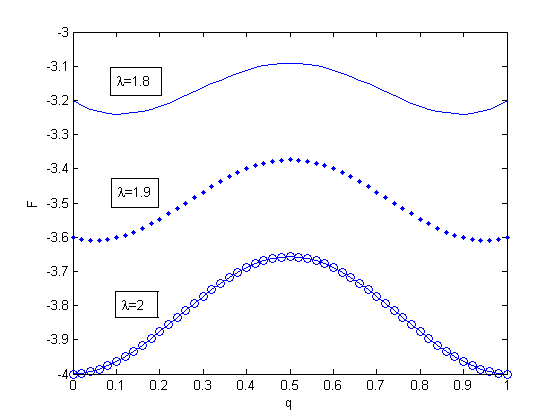}}
\caption{The function $F(q)$ as penalty parameter $\lambda$ varies from 1.8 (solid line, top) to 1.9 (dot, middle) and 2 (line-circle, bottom) where minimal points occur at $q=0, 1$. Interior minimal points occur on $[0,1]$ when $\lambda = 1,8, 1.9$.}
\label{gmpic2}    
\end{center}
\end{figure}

\begin{rem}
We refer to \cite{anticvx} on certain correlated random Bernoulli graphs where $Q^* \not = Q_*$. 
On the other hand, there is a class of friendly graphs \cite{Kimm_15} where $Q^*=Q_*$. 
Existing techniques to improve convex relaxation on GM and QAP include approximate quadratic programming, sorting networks and path following based homotopy methods  \cite{vog,wright_16,path_09}. Our proposed penalty (\ref{DL12})-(\ref{ct}) appears more direct and generic. A detailed comparison will be left for a future study.
\end{rem}

\begin{rem}
In example 1 above, if the convex relaxed $q_*=2/3$ is rounded up to 1, then $Q_* = Q^*$. In example 2 (Fig. \ref{gmpic2}), the two interior minimal points at $\lambda = 1.8,  1.9$, after rounding down (up), become zero or one. So convex relaxation with the help of rounding still recovers the exact permutation.  
We show in example 3 below that convex relaxation still fails after rounding (to 1 if the number is above 1/2, to 0 if the number is below 1/2).
\end{rem}

{\bf Example 3:} We consider the two-layer neural network model with one hidden layer \cite{LiYuan}. Given $m \geq 0$, the forward model is the following function:
\begin{align*}
    f_m\left(x,W\right) = \left\|\, \phi\left(\left(mI+W\right)x\right)\, \right\|_1,
\end{align*}
where $\phi\left(v\right) = \max\left(v,0\right)$ is the ReLU activation function, $x =(x_1,x_2) \in \mathbf{R}^2$ is the input random vector drawn from a probability distribution, $W \in \mathbf{R}^{2\times2}$ is the weight matrix, $I$ is the identity matrix. Consider a two-layer teacher network with $2\times 2$ weight matrix
\begin{align*}
    W^* = \left[\begin{array}{cc}
         a & b \\
         c & d
    \end{array}\right],\quad a, b, c, d \geq 0.
\end{align*}
We train the student network with doubly stochastic constraint on $W$ using the $\ell_2$ loss:
\begin{align*}
   L\left(W\right) = \mathbf{E}_x\left[f_m\left(x,W\right)-f_m\left(x,W^*\right)\right]^2. 
\end{align*}
Let $p \in \left[0,1\right]$,
\begin{align*}
    W = \left[\begin{array}{cc}
         p & 1-p \\
         1-p & p
    \end{array}\right].
\end{align*}
We write the loss function as
\begin{align*}
    \l_m\left(p\right) := &L\left(W\right)\\
    = &\, \mathbf{E}_x\left[\phi\left(\left(m+p\right)x_1+\left(1-p\right)x_2\right)+\phi\left(\left(1-p\right)x_1+\left(m+p\right)px_2\right)\right.\\
    & \left.-\phi\left(ax_1+bx_2\right)-\phi\left(cx_1+dx_2\right)\right]^2\\
    = &\, \mathbf{E}_x\phi\left(\left(m+p\right)x_1+\left(1-p\right)x_2\right)^2+\mathbf{E}_x\phi\left(\left(1-p\right)x_1+\left(m+p\right)x_2\right)^2\\
    &+2\mathbf{E}_x\left[\phi\left(\left(m+p\right)x_1+\left(1-p\right)x_2\right)\phi\left(\left(1-p\right)x_1+\left(m+p\right)x_2\right)\right]\\
    &-2G_m\left(p,a,b\right)-2G_m\left(p,c,d\right)+\mathbf{E}_x\left[\phi\left(ax_1+bx_2\right)+\phi\left(cx_1+dx_2\right)\right]^2, \numberthis \label{eqn3}
\end{align*}
where for $s, t \geq 0$, $G_m\left(p,s,t\right)$ is defined as
\begin{align*}
    \mathbf{E}_x\left[\phi\left(\left(m+p\right)x_1+\left(1-p\right)x_2\right)\phi\left(sx_1+tx_2\right)+\phi\left(\left(1-p\right)x_1+\left(m+p\right)x_2\right)\phi\left(sx_1+tx_2\right)\right].
\end{align*}
Define $\mathcal{I}\left(q,r,s,t\right) := \mathbf{E}_x\left[\phi\left(qx_1+rx_2\right)\phi\left(sx_1+tx_2\right)\right],$ then $\mathcal{I}\left(q,r,s,t\right) = \mathcal{I}\left(s,t,q,r\right)$ and
\begin{align*}
    G_m\left(p,s,t\right) = \mathcal{I}\left(m\!+\!p,1\!-\!p,s,t\right)+\mathcal{I}\left(1\!-\!p,m\!+\!p,s,t\right).
\end{align*}
For simplicity, let $x$ obey uniform distribution on $\left[-1,1\right]^2$. For $qt \geq rs$, $q+r > 0$, $s+t > 0$, $\mathcal{I}\left(q,r,s,t\right)$ equals
\begin{align*}
    \left\{\begin{aligned}
    &\dfrac{2}{3}\left(qs+rt\right)+\dfrac{q^2\left(qt-3rs\right)}{24r^2}+\dfrac{s^2\left(3qt-rs\right)}{24t^2}, q < r\\
    &\dfrac{1}{3}\left(qs+rt\right)+\dfrac{1}{4}\left(qt+rs\right)+\dfrac{1}{24}(\dfrac{r^2}{q^2}+\dfrac{s^2}{t^2})\left(3qt-rs\right),\ q \geq r \text{ and } t \geq s\\
    &\dfrac{2}{3}\left(qs+rt\right)+\dfrac{r^2\left(3qt-rs\right)}{24q^2}+\dfrac{t^2\left(qt-3rs\right)}{24s^2}, t < s.
    \end{aligned}\right. \numberthis \label{eqn_i}
\end{align*}
We have
\begin{align*}
\mathbf{E}_x\phi\left(\left(m+p\right)x_1+\left(1-p\right)x_2\right)^2 &= \mathbf{E}_x\phi\left(\left(1-p\right)x_1+\left(m+p\right)x_2\right)^2\\
&= \dfrac{2}{3}\left[\left(m+p\right)^2+\left(1-p\right)^2\right], \numberthis \label{eqn3-1}
\end{align*}
\begin{align*}
\mathbf{E}_x\left[\phi\left((m\!+\!p)x_1\!+\!(1\!-\!p)x_2\right)\phi\left((1\!-\!p)x_1\!+\!(m\!+\!p)x_2\right)\right] = \dfrac{\left(m+1\right)^3}{3\theta_m\left(p\right)}+\dfrac{\left(m+p\right)^4}{12\theta_m\left(p\right)^2}, \numberthis \label{eqn3-2}
\end{align*}
where $\theta_m\left(p\right) := \max\left(m+p,1-p\right)$.
The last term in \eqref{eqn3} is a constant:
\begin{align*}
\mathbf{E}_x\left[\phi\left(ax_1+bx_2\right)+\phi\left(cx_1+dx_2\right)\right]^2 = \dfrac{2}{3}\left(a^2+b^2+c^2+d^2\right)+2\, \mathcal{I}\, \left(a,b,c,d\right).\numberthis \label{eqn3-4}
\end{align*}
Consider a special case when $a = 1/3$, $b = 2/3$, $c = 1/4$ and $d = 3/4$. By \eqref{eqn3-1}-\eqref{eqn3-4}, the loss function $l_m\left(p\right)$ equals
\begin{align*}
\dfrac{2}{3}\left[\left(m+p\right)^2+\left(1-p\right)^2\right]+\dfrac{2\left(m+1\right)^3}{3\theta_m\left(p\right)}+\dfrac{\left(m+1\right)^4}{6\theta_m\left(p\right)^2}-2G_m(p,\dfrac{1}{3},\dfrac{2}{3})-2G_m(p,\dfrac{1}{4},\dfrac{3}{4})+\dfrac{8113}{5184}.
\end{align*}

Let $F_m\left(p\right) := l_m\left(p\right)-4\lambda\sqrt{p^2+\left((1-p\right)^2}.$ When $m= 0$, $\lambda =0$, Fig. \ref{2layer_uniform} (top left) shows $l_0\left(p\right)$ has minimal points in the interior of $\left(0, 1\right)$. A permutation matrix $W$ that minimizes $L\left(W\right)$ can be achieved by rounding the minimal points. However, when $m = 1$, $\lambda=0$, (Fig. \ref{2layer_uniform}, top right), rounding the interior minimal point of $l_1\left(p\right)$ gives the wrong permutation matrix at $p = 1$. At $\lambda=0.4$, the $P$ regularization selects the correct permutation matrix.

\begin{figure}[!ht]
\begin{center}
\begin{tabular}{cc}
\includegraphics[width=0.4\textwidth]{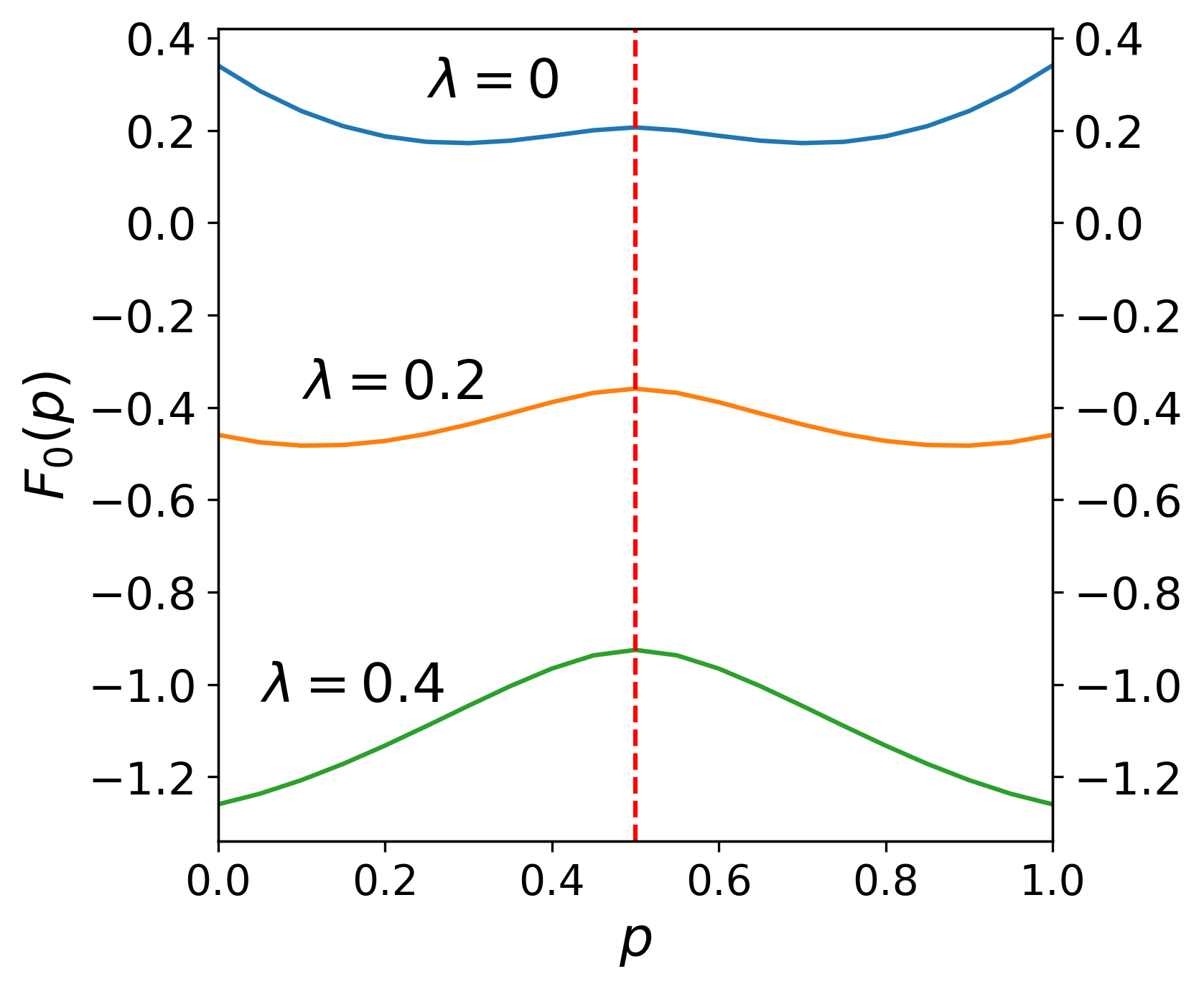} &
\includegraphics[width=0.43\textwidth]{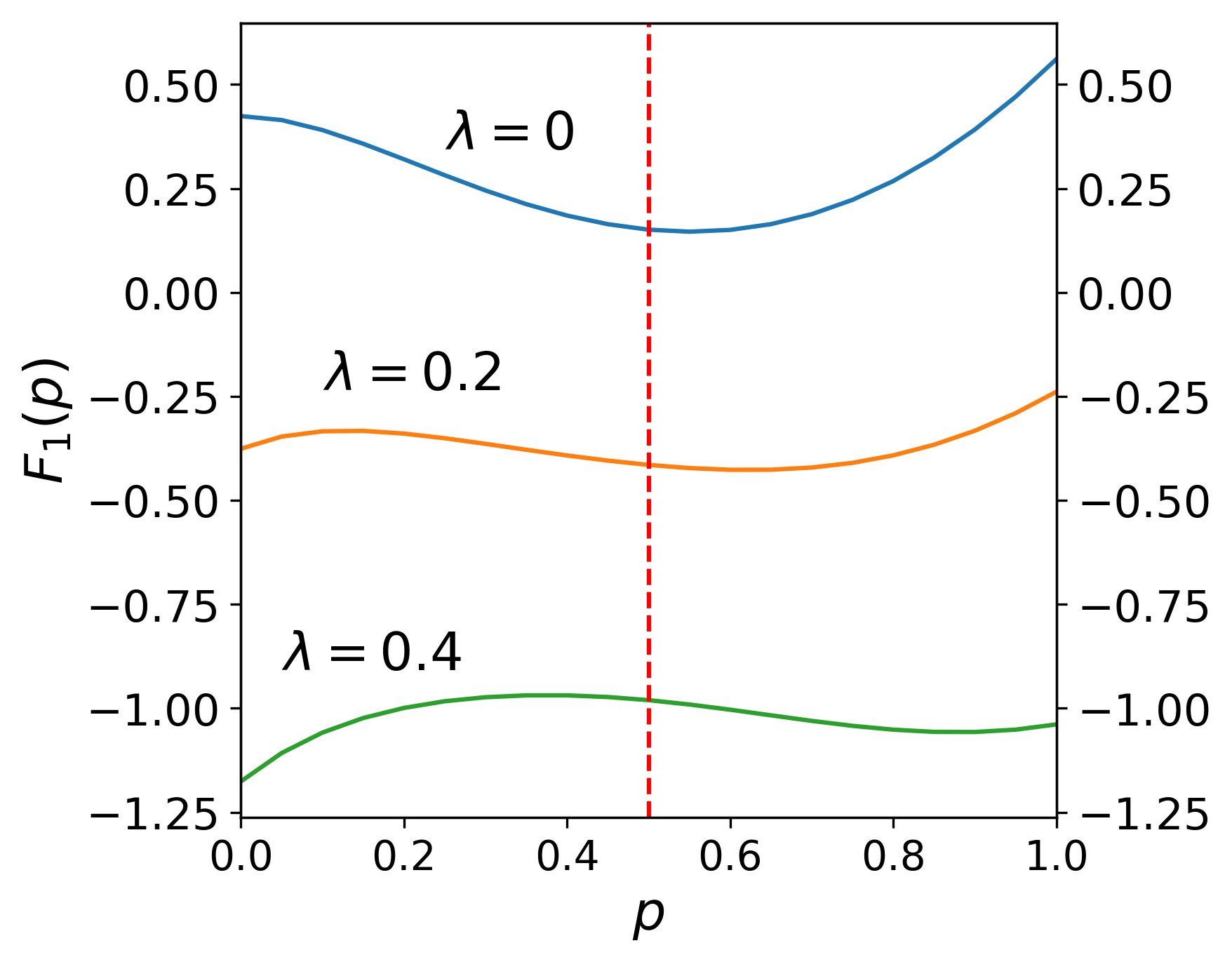}
\end{tabular}
\caption{$F_m(p)$ ($m=0$ left, $m=1$  right) as penalty parameter $\lambda$ varies for the uniformly distributed input data on $[-1, 1]^2$.}
\label{2layer_uniform}    
\end{center}
\end{figure}

\begin{rem}
If $x$ obeys the unit Gaussian distribution as in \cite{LiYuan}, the $F_m(p)$ functions are more complicated analytically, however their plots resemble those for uniformly distributed $x$, see Fig. \ref{2layer_gaussian}. 
\end{rem}



\begin{figure}[!ht]
\begin{center}
\begin{tabular}{cc}
    \includegraphics[width=0.41\textwidth]{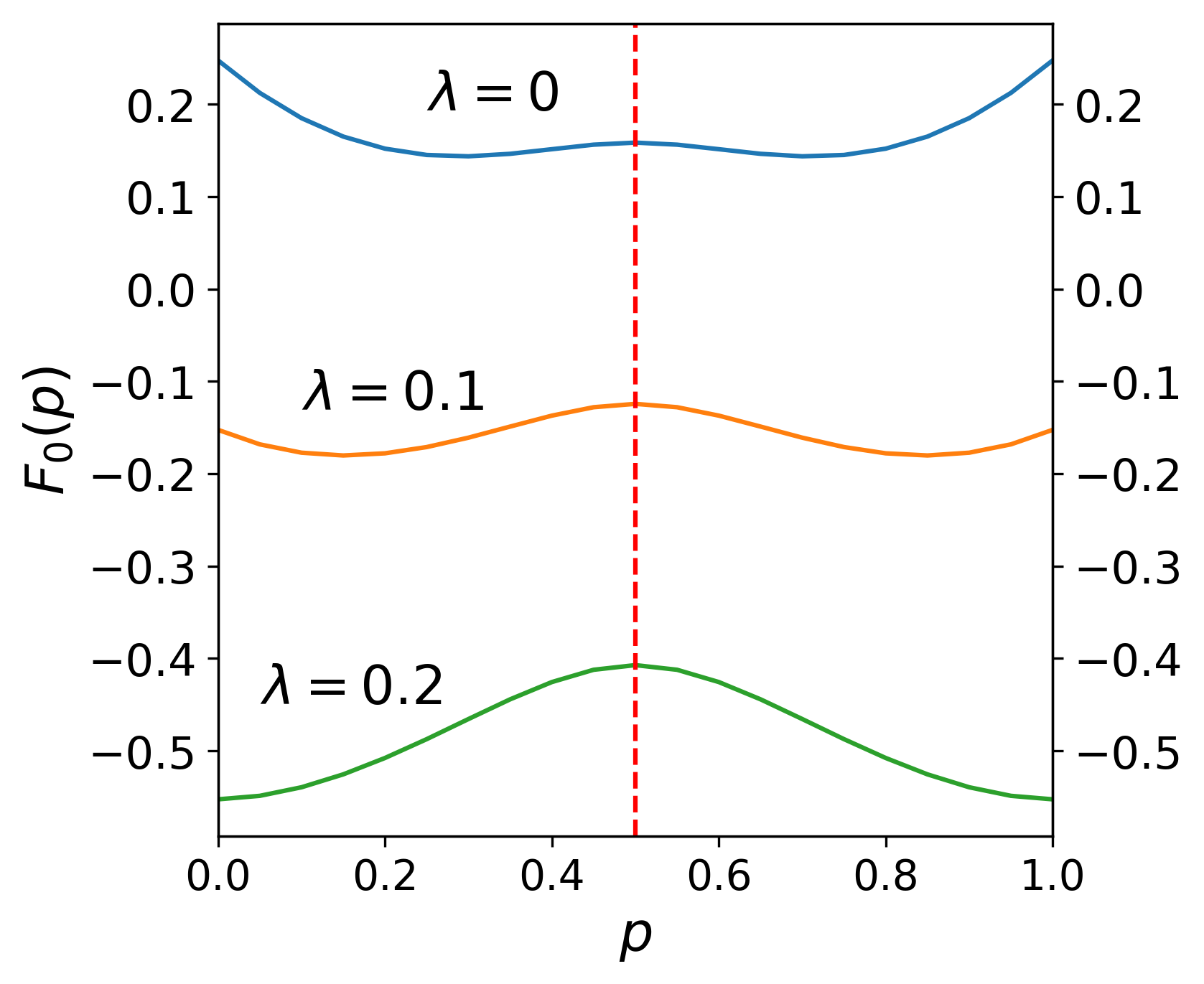} & \includegraphics[width=0.41\textwidth]{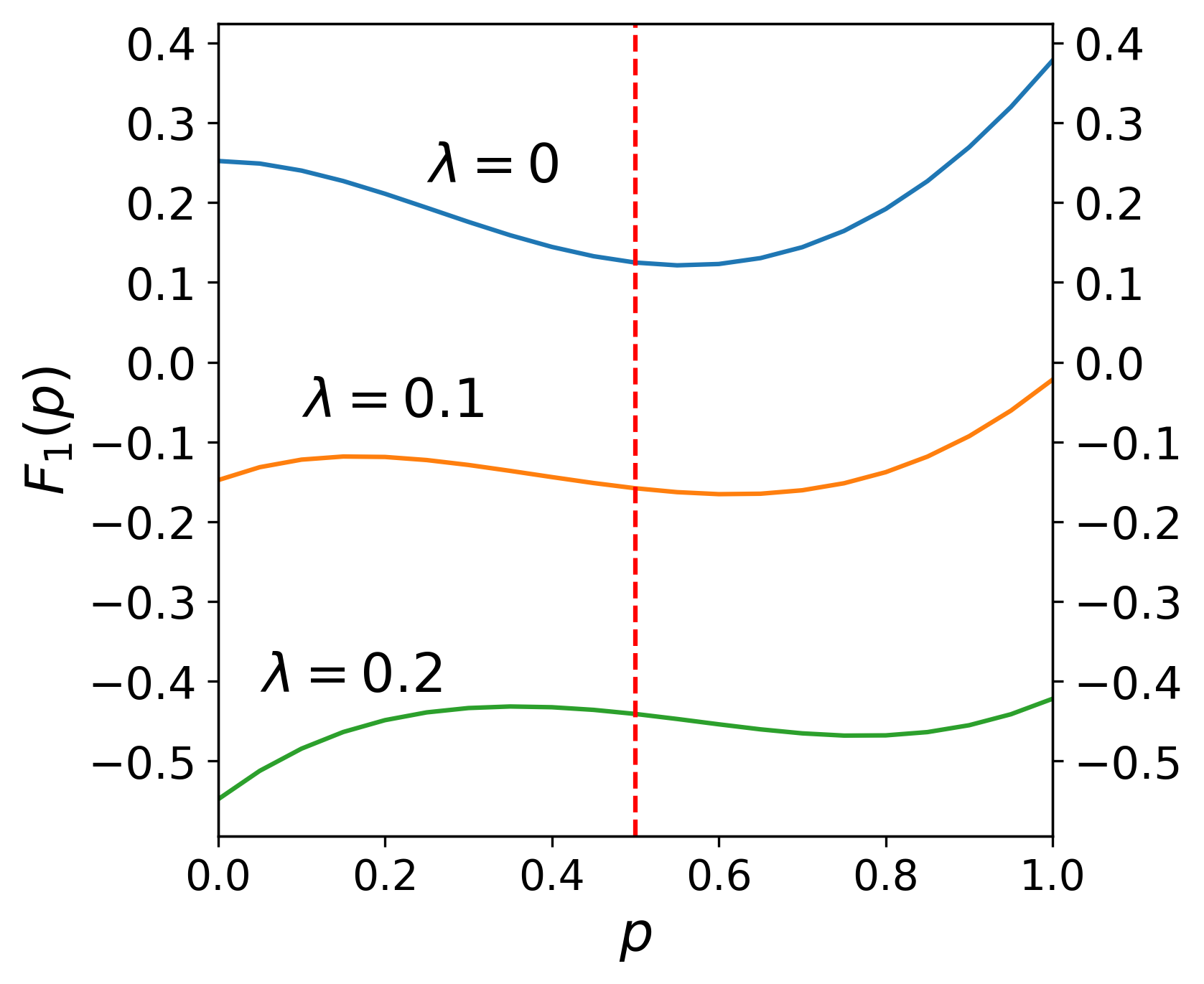}
\end{tabular}
\caption{$F_m(p)$ ($m=0$ left, $m=1$ right) as penalty parameter $\lambda$ varies for unit Gaussian input data on $\mathbf{R}^2$.}
\label{2layer_gaussian}    
\end{center}
\vskip -0.2in
\end{figure}

\section{Experiments}

We relax the shuffle units in ShuffleNet v1 \cite{shufflev1} and ShuffleNet v2 \cite{shufflev2} and perform experiments on CIFAR-10 \cite{cifar_09} and ImageNet \cite{imagenet_09, imagnet_12} classification datasets. The accuracy results of auto shuffles are evaluated after the relaxed shuffles are rounded.

On CIFAR-10 dataset, we set the $\ell_{1-2}$ penalty parameter $\lambda = 10^{-3}$. All experiments are randomly initialized with learning rate linearly decaying from $0.2$. We train each network for $200$ epochs. We set weight-decay $10^{-4}$, momentum $0.95$ and batch size $128$. The experiments are carried out on a machine with single Nvidia GeForce GTX 1080 Ti GPU. In Table 1, we see that auto shuffle consistently improves by as much as 1\% on manual shuffle in both v1 and v2 models of ShuffleNet.


\begin{table}[!ht]
\caption{CIFAR-10 validation accuracies.}
\label{tab:1}
\begin{center}
\begin{tabular}{lcc}
  \toprule
  Networks & Manual & Auto \\
  \midrule
  ShuffleNet v1 (g=3) & 90.55 & 91.76\\
  ShuffleNet v1 (g=8) & 90.06 & 91.26\\
  ShuffleNet v2 (1$\times$) & 91.90 & 92.81\\
  ShuffleNet v2 (1.5$\times$) & 92.56 & 93.22\\
  \bottomrule
\end{tabular}
\end{center}
\end{table}

On ImageNet dataset, we set the $\ell_{1-2}$ penalty parameter $\lambda = 2\times10^{-3}$. For each experiment, the training process includes two training cycles: the first cycle is randomly initialized with learning rate starting at $0.2$ and the second one is resumed from the first one with learning rate starting at $0.1$. Each cycle consists of $100$ epochs and the learning rate decays linearly. We set weight-decay $4\times10^{-5}$, momentum $0.9$ and batch size $512$. The experiments are carried out on a machine with 4 Nvidia GeForce GTX 1080 Ti GPUs. In Table \ref{tab:2}, we see that auto shuffle again consistently improves on manual shuffle for both v1 and v2 models.

\begin{table}[!ht]
\caption{ImageNet top-1 validation accuracies.}
\label{tab:2}
\begin{center}
\begin{tabular}{lcc}
  \toprule
  Networks & Manual & Auto \\
  \midrule
  ShuffleNet v1 (g=3) & 65.50 & 65.62\\
  ShuffleNet v1 (g=8) & 65.18 & 65.76\\
  ShuffleNet v2 (1$\times$) & 67.46 & 67.69\\
  ShuffleNet v2 (1.5$\times$) & 70.58 & 70.60\\
  \bottomrule
\end{tabular}
\end{center}
\end{table}

The permutation matrix of the first shuffle unit in ShuffleNet v1 (g=3) is a matrix of size $60\times60$, which can be visualized in Fig. \ref{pmat} (manual, left) along with an auto shuffle (right). The dots (blanks) denote locations of 1's (0's). The auto shuffle looks disordered while the manual shuffle is ordered. However, the inference cost of auto shuffle is comparable to manual shuffle since the shuffle is fixed and stored after training.

\begin{figure}[!ht]
\begin{center}
\begin{tabular}{cc}
\includegraphics[width=0.35\textwidth]{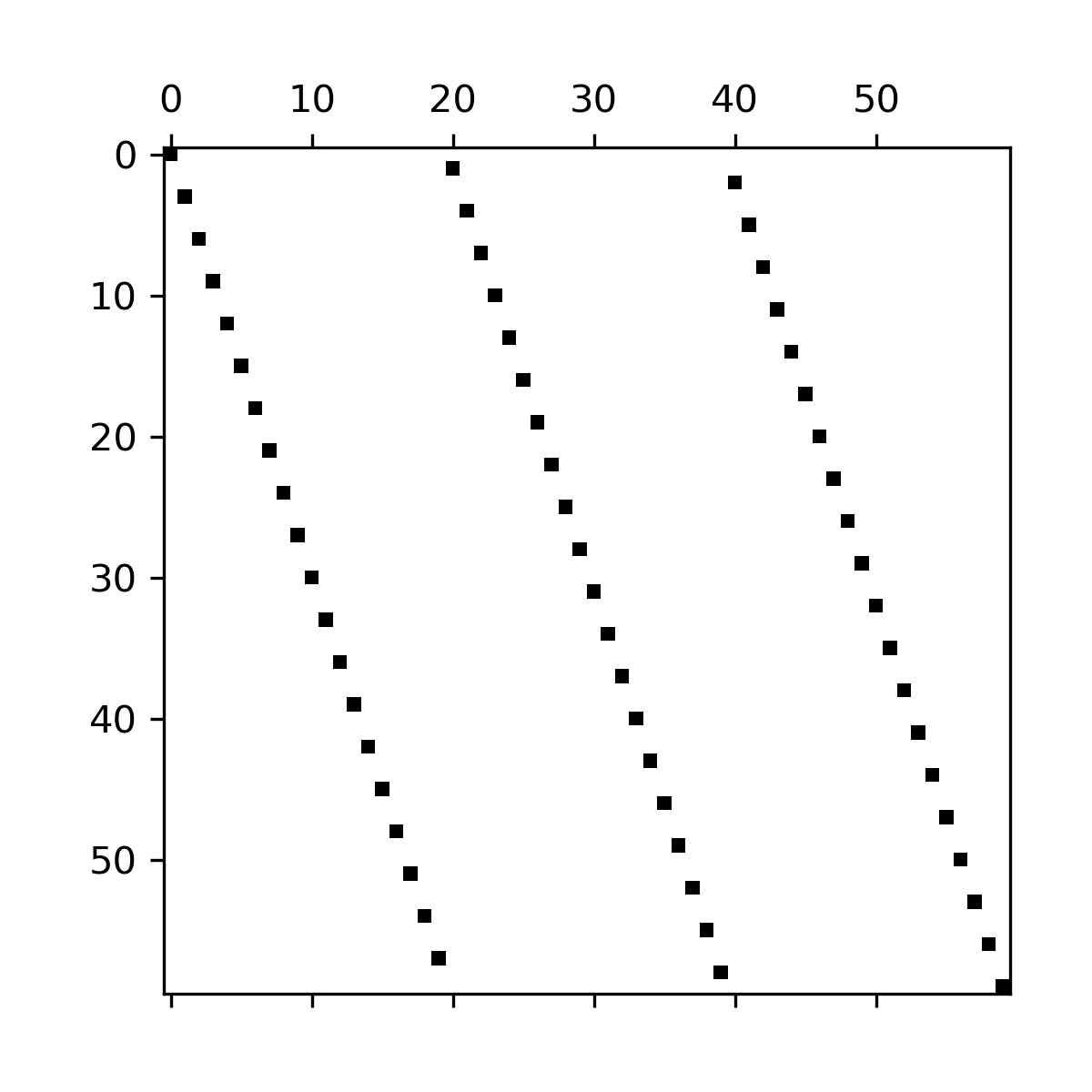}
\includegraphics[width=0.35\textwidth]{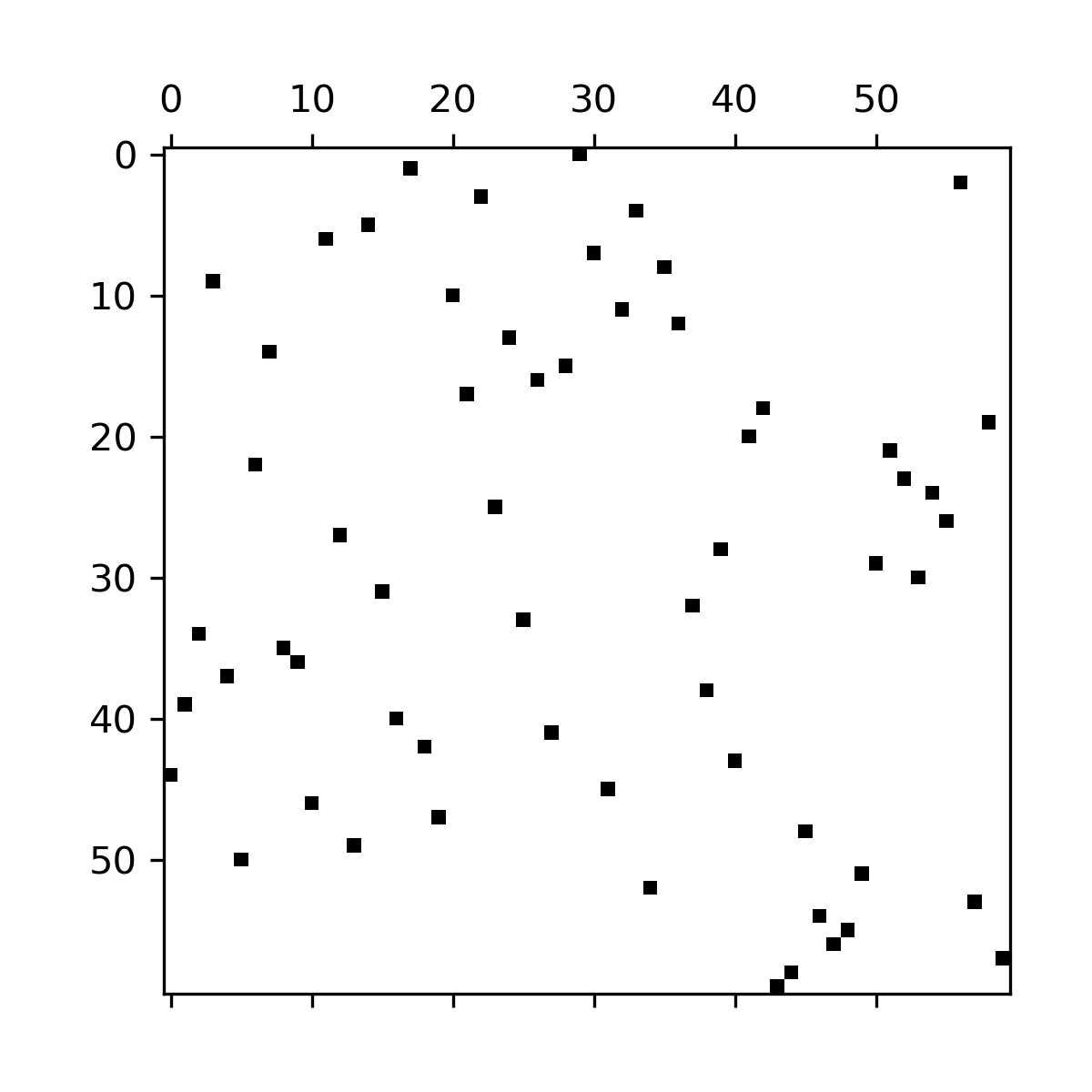}
\end{tabular}
\caption{Permutation matrices of the first shuffle unit in ShuffleNet v1 (g=3) of manual shuffle (left) and auto shuffle (right). The auto shuffle is trained on CIFAR-10 dataset. The dots (blanks) indicate locations of 1's (0's). The auto shuffle looks disordered while the manual shuffle is ordered. However, the inference cost of auto shuffle is comparable to manual shuffle in inference. }
\label{pmat}
\end{center}
\end{figure}

The accuracy drop due to rounding to produce auto shuffle from relaxed shuffle is indicated by relative error in Table \ref{tab:3}. On CIFAR-10 dataset, negligible drop is observed for ShuffleNet v1. Interestingly, rounding even gained accuracy for ShuffleNet v1 on ImageNet dataset.

\begin{table}[!ht]
\caption{Relative error of accuracy of rounding relaxed shuffle. The -/+ refer to accuracy drop/gain after rounding to produce auto shuffle from relaxed shuffle.}
\label{tab:3}
\begin{center}
\begin{tabular}{llc}
  \toprule
  Dataset & Networks & Relative Error\\
  \midrule
  \multirow{4}{*}{CIFAR-10} & ShuffleNet v1 (g=3) & 0\\
  & ShuffleNet v1 (g=8) & 0\\
  & ShuffleNet v2 (1$\times$) & -2.15E-4\\
  & ShuffleNet v2 (1.5$\times$) & -1.07E-3\\
  \midrule
  \multirow{4}{*}{ImageNet} & ShuffleNet v1 (g=3) & +6.10E-5\\
  & ShuffleNet v1 (g=8) & +3.04E-5\\
  & ShuffleNet v2 (1$\times$) & 0\\
  & ShuffleNet v2 (1.5$\times$) & -2.83E-5\\
  \bottomrule
\end{tabular}
\end{center}
\end{table}

The $\ell_{1-2}$ penalty of ShuffleNet v1 (g=3) is plotted in Fig. \ref{penaltyplot}. As the penalty decays, the validation accuracy of \textbf{auto shuffle} (after rounding) becomes closer to \textbf{relaxed shuffle} (before rounding), see Fig. \ref{roundplot}.

\begin{figure}[!ht]
\begin{center}
\centerline{\includegraphics[width=0.65\textwidth]{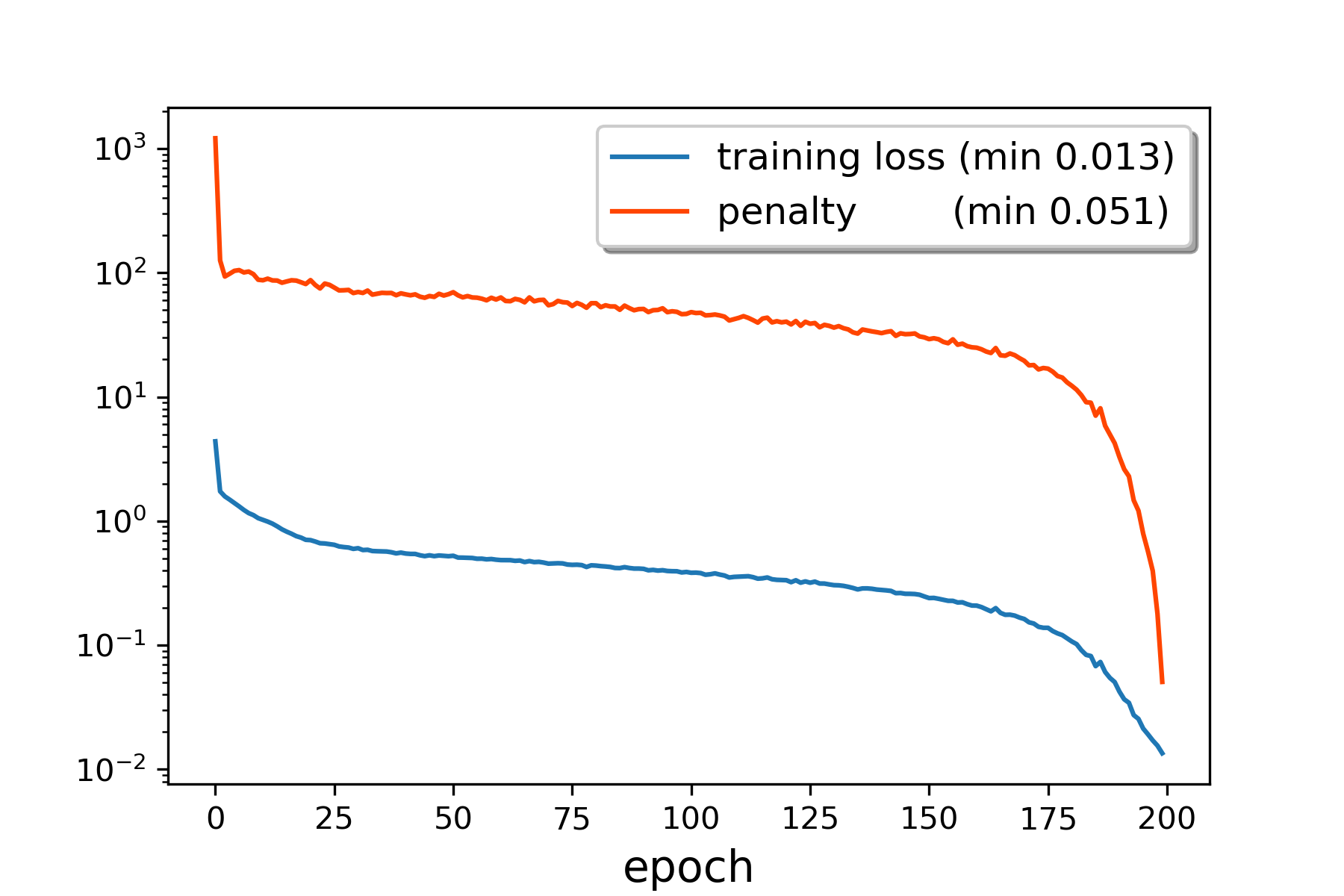}}
\caption{Training loss $L$ and penalty $P$ of ShuffleNet v1 (g=3) with relaxed shuffle on CIFAR-10.}
\label{penaltyplot}
\end{center}
\end{figure}

\begin{figure}[!ht]
\begin{center}
\centerline{\includegraphics[width=0.65\textwidth]{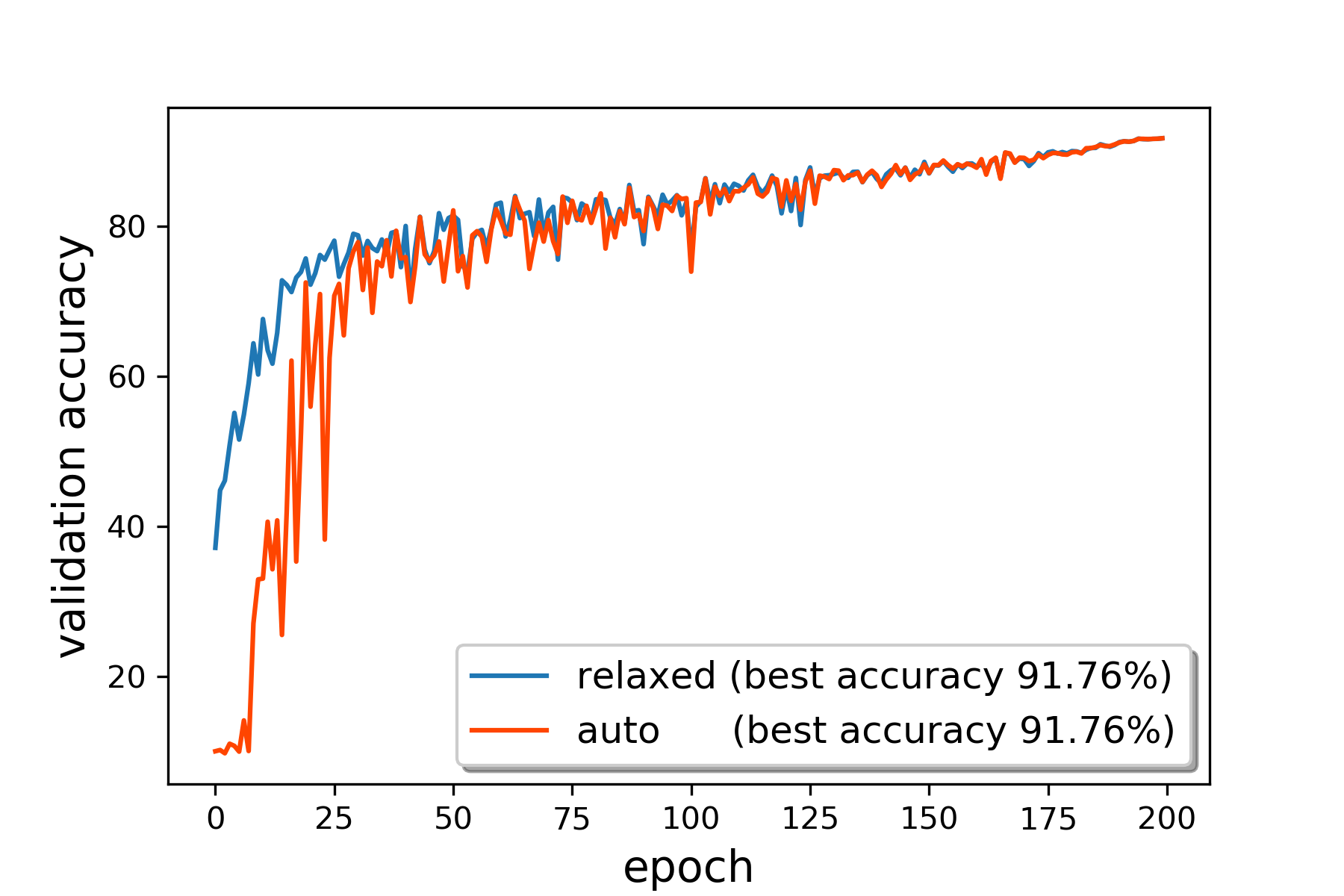}}
\caption{Validation accuracy of ShuffleNet v1 (g=3) with relaxed shuffle (before rounding) and auto shuffle (after rounding) on CIFAR-10. The rounding error becomes smaller during training.}
\label{roundplot}
\end{center}
\end{figure}

To demonstrate the significance of the $\ell_{1-2}$ regularization, we also tested auto shuffle with various $\lambda$ on ShuffleNet v1 (g=3). Table \ref{tab:lambda} shows that the accuracy drops much after the relaxed shuffle is rounded. We plot the stochastic matrix of the first shuffle unit of the network at $\lambda = 0$ and $\lambda = 10^{-5}$ respectively in Fig. \ref{smat}. The penalty is large when $\lambda$ is relatively small, indicating that the stochastic matrices learned are not close to optimal permutation matrices.

\begin{table}[!ht]
\caption{CIFAR-10 validation accuracies of ShuffleNet v1 (g=3) with relaxed shuffle (before rounding) and auto shuffle (after rounding), and penalty values of relaxed shuffle at various $\lambda$'s. The penalty and rounding error tends to zero as $\lambda$ increases.}
\label{tab:lambda}
\begin{center}
\begin{tabular}{cccccc}
  \toprule
  $\lambda$ & 0 & 1E-5 & 1E-4 & 5E-4 & 1E-3\\
  \midrule
  relaxed & 90.00 & 90.18 & 90.48 & 91.45 & 91.76\\
  auto & 10.00 & 38.18 & 11.37 & 71.50 & 91.76\\
  penalty & 3.37E3 & 1.59E3 & 4.95E2 & 3.13E-1 & 5.07E-2\\
  \bottomrule
\end{tabular}
\end{center}
\end{table}

\begin{figure}[!ht]
\begin{center}
\begin{tabular}{cc}
\includegraphics[width=0.35\textwidth]{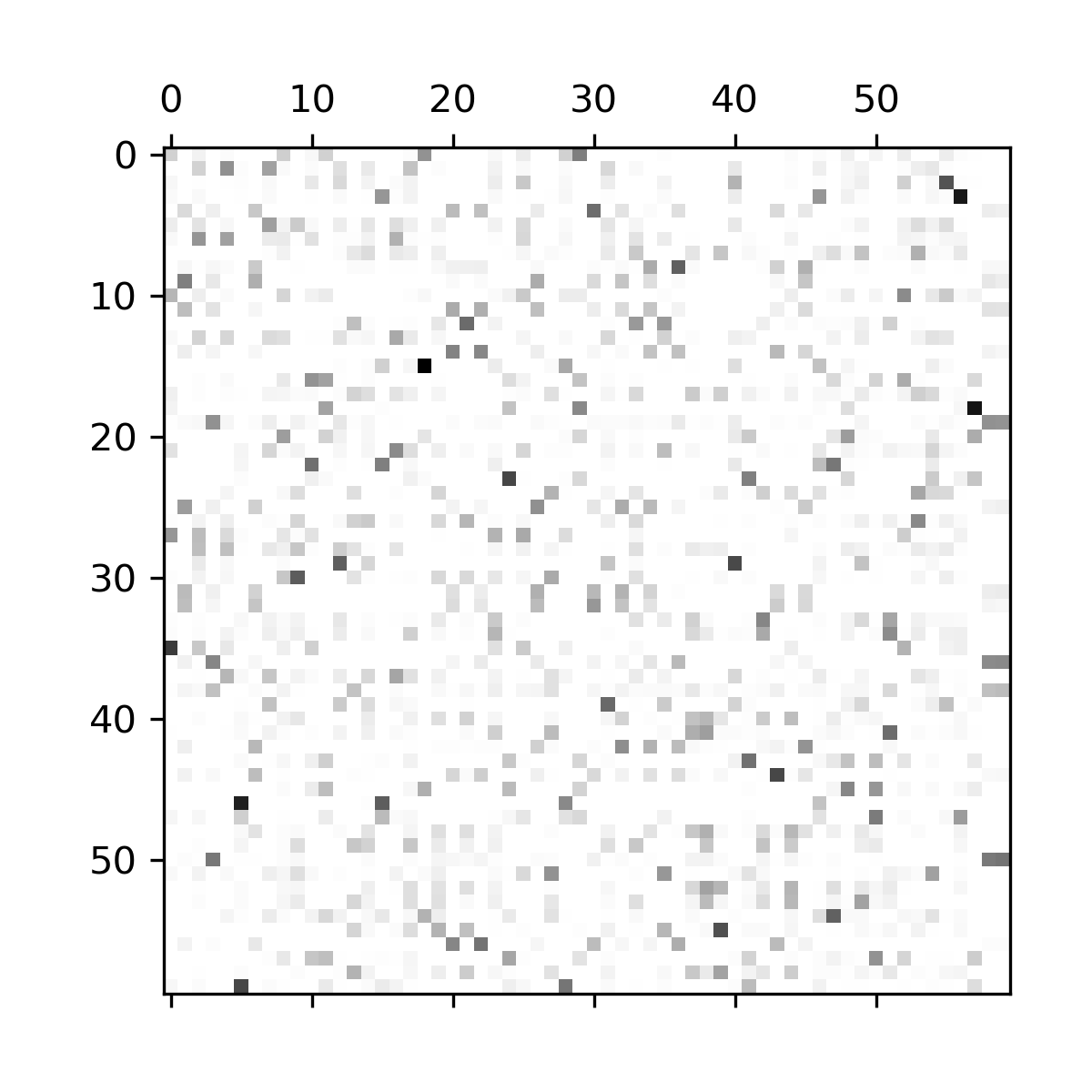}
\includegraphics[width=0.35\textwidth]{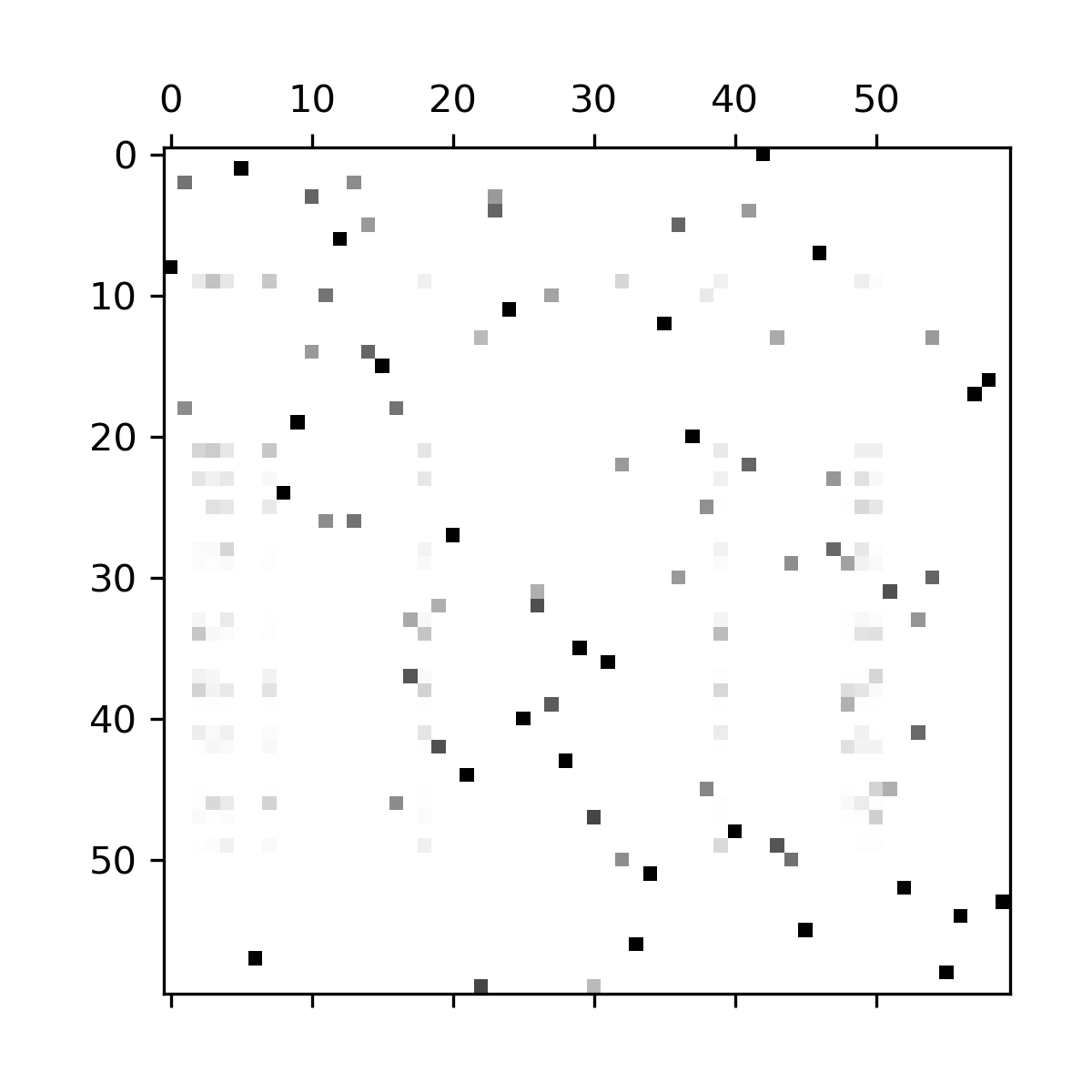}
\end{tabular}
\caption{Stochastic matrices of the first shuffle unit in ShuffleNet v1 (g=3) with relaxed shuffle before rounding at $\lambda = 0$ (left) and $\lambda = 10^{-5}$ (right). The relaxed shuffle is trained on CIFAR-10 dataset. The matrices are quite diffusive and not close to optimal permutation matrices when $\lambda$ is relatively small.}
\label{smat}
\end{center}
\end{figure}

\section{Conclusion}
We introduced a novel, exact and Lipschitz continuous relaxation for permutation and learning channel shuffling in ShuffleNet.
We showed through a regression problem of a 2-layer neural network with short cut that convex relaxation fails even with additional rounding while our relaxation is precise. 
We plan to extend our work to auto-shuffling of other LCNNs and hard permutation problems in the future.    

\section{Acknowledgement} 
The work was partially supported 
by NSF grant IIS-1632935.

\clearpage

\bibliographystyle{plain}
\bibliography{references}

\end{document}